\documentclass[sigconf]{acmart}
\usepackage{subcaption}

\AtBeginDocument{%
  }

  \copyrightyear{2026}
\acmYear{2026}
\setcopyright{cc}
\setcctype{by}
\acmConference[CIKM '26] {Proceedings of the 35th ACM International Conference on Information and Knowledge Management}{November 7--11, 2026}{Rome, Italy.}
\acmBooktitle{Proceedings of the 35th ACM International Conference on Information and Knowledge Management (CIKM '26), November 7--11, 2026, Rome, Italy}
\acmISBN{979-8-4007-2539-5/2026/11}
\acmDOI{10.1145/3799682.3840846}

\usepackage[symbol]{footmisc}
\usepackage{soul}
\usepackage{amsthm}
\usepackage{bm}
\usepackage{algorithm}
\usepackage{algorithmic}

\usepackage{cleveref}
\usepackage{multirow}
\usepackage[most]{tcolorbox}

\usepackage{balance}

\definecolor{inputbg}{HTML}{F5F5F5}
\definecolor{inputframe}{HTML}{BDBDBD}
\definecolor{outputbg}{HTML}{E8F5E9}
\definecolor{outputframe}{HTML}{81C784}
\definecolor{errorbg}{HTML}{FFF3E0}
\definecolor{errorframe}{HTML}{FFB74D}

\newtcolorbox{inputbox}{
  colback=inputbg, colframe=inputframe,
  fonttitle=\bfseries\small, title=Input,
  boxrule=0.5pt, arc=1.5pt, left=4pt, right=4pt, top=2pt, bottom=2pt,
  before skip=4pt, after skip=0pt
}
\newtcolorbox{outputbox}[1][FISolver]{
  colback=outputbg, colframe=outputframe,
  fonttitle=\bfseries\small, title={Output (#1)},
  boxrule=0.5pt, arc=1.5pt, left=4pt, right=4pt, top=2pt, bottom=2pt,
  before skip=0pt, after skip=4pt
}
\newtcolorbox{errorbox}[1][SFT-only]{
  colback=errorbg, colframe=errorframe,
  fonttitle=\bfseries\small, title={Output (#1)},
  boxrule=0.5pt, arc=1.5pt, left=4pt, right=4pt, top=2pt, bottom=2pt,
  before skip=0pt, after skip=0pt
}

\newtheorem{example}{Example}
\newtheorem{theorem}{Theorem}
\newtheorem{definition}{Definition}
\newtheorem{proposition}{Proposition}
\theoremstyle{remark}
\newtheorem*{remark}{Remark}

\begin{document}

\title{FISolver: Data-Driven Discovery of First Integrals via Backward Generation and Guided Reinforcement Learning}


\author{Jingfeng Zhong}
\authornote{Both authors contributed equally to this research.}
\affiliation{%
  \institution{Shanghai Jiao Tong University}
  \city{Shanghai}
  \country{China}}
\email{zhongjingfeng@sjtu.edu.cn}

\author{Zhengxiang Liu}
\authornotemark[1]
\affiliation{%
  \institution{Shanghai Jiao Tong University}
  \city{Shanghai}
  \country{China}}
\email{2312471604@sjtu.edu.cn}

\author{Zhijie Wang}
\affiliation{%
  \institution{Shanghai Jiao Tong University}
  \city{Shanghai}
  \country{China}}
\email{violetevergarden@sjtu.edu.cn}

\author{Shuai Li}
\authornote{Corresponding author.}
\affiliation{%
  \institution{Shanghai Jiao Tong University}
  \institution{Shanghai Innovation Institute}
  \city{Shanghai}
  \country{China}}
\email{shuaili8@sjtu.edu.cn}

\renewcommand{\shortauthors}{Jingfeng Zhong, Zhengxiang Liu, Zhijie Wang, and Shuai Li}

\begin{abstract}
The discovery of first integrals is of profound scientific significance for understanding physical conservation laws in dynamical systems. However, existing symbolic computation systems and Large Language Models (LLMs) struggle with this task due to the scarcity of high-quality training data and the reliance on mathematical intuition for finding solutions. This paper introduces \textbf{FISolver}, an LLM-based solver designed to bridge this gap. First, we propose a novel ``Backward Generation'' algorithm that systematically constructs large-scale datasets of (differential equations, first integrals) pairs by deriving equations from sampled integrals, effectively overcoming the data scarcity bottleneck. Second, we perform supervised fine-tuning on a compact mathematical model and further enhance its capabilities using reinforcement learning with a Levenshtein Distance-based shaped reward function. Additionally, we introduce data synthesis and blending techniques that enable effective adaptation to specific hard instances with sparse examples. Experiments demonstrate that \textbf{FISolver}, with significantly lower computational cost, substantially outperforms larger mathematical LLMs and commercial solvers like Mathematica on complex benchmarks, establishing a new data-driven paradigm for automated mathematical discovery.
\end{abstract}



\maketitle

\section{Introduction}

The discovery of first integrals is of profound scientific significance. In dynamical systems, a first integral is a conserved quantity whose total derivative with respect to time is zero along the system's trajectories~\cite{arnold1992ordinary}. They provide implicit solutions, reduce the system order, and are directly linked to physical conservation laws~\cite{arnold1992ordinary}. Furthermore, they are central to modeling physical systems and related to breakthroughs in fundamental mathematical problems like Hilbert's sixteenth problem~\cite{ilyashenko2002centennial} and the three-body problem~\cite{poincare1893methodes}.

However, identifying first integrals remains a formidable challenge. The discovery process often relies on researchers' empirical insights rather than explicit logic. More precisely, finding a first integral $V$ requires searching over a large symbolic function space while satisfying the differential constraint $\frac{dV}{dt}=0$. Despite theoretical progress, existing computer algebra systems~\cite{braz2025new} struggle with this symbolic computation task because many algorithms rely on restricted ansatzes or search strategies that are effective for selected families but difficult to scale to more diverse expressions.

While the capabilities of Large Language Models (LLMs) in solving mathematical problems have advanced significantly, with some models rivaling human experts on specific benchmarks~\cite{yang2024formal,lewkowycz2022solving,trinh2024solving}, they still encounter difficulties with problems of great complexity or those lacking a fixed solution paradigm~\cite{yu2023metamath,lample2019deep}. First-integral discovery is one such problem: its solution process is difficult to formalize into explicit logical chains~\cite{davies2021advancing}, and large-scale verified training datasets are critically scarce~\cite{lample2019deep}. Nevertheless, pre-trained mathematical LLMs remain a suitable foundation for this task because they provide useful symbolic priors over mathematical notation, operator composition, and formula-like sequences, making them more promising than training a small task-specific generator from scratch. Consequently, the key question is not only whether LLMs can be applied to first-integral discovery, but how to provide them with sufficient correct symbolic supervision and a reliable verification mechanism.

To address these challenges, we introduce \textbf{FISolver}, an LLM-based solver designed to bridge this gap by establishing a new data-driven predict-and-verify paradigm. We demonstrate that fine-tuning specialized LLMs with sufficient verified data is a promising direction to directly identify first integrals. The central enabling technique is Backward Generation: instead of solving a differential system for an unknown integral, we sample first integrals and algebraically derive compatible differential equations. As shown in Figure~\ref{fig:main_results_bar}, FISolver substantially outperforms both larger mathematical LLMs and the commercial solver Mathematica on complex benchmarks, particularly on hard nonlinear systems where existing tools struggle. This work highlights the substantial potential of large language models in tackling complex symbolic mathematics when they are combined with scalable data generation and symbolic verification.

\begin{figure}[t]
    \centering
    \includegraphics[width=\linewidth]{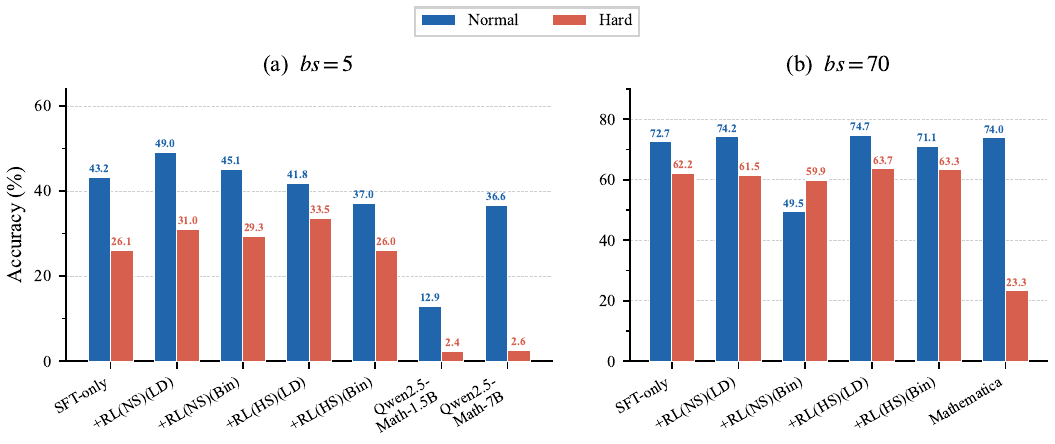}
    \caption{Accuracy of first integral prediction on the Normal and Hard test sets. FISolver (1.5B parameters) significantly outperforms the base model Qwen2.5-Math-1.5B, the larger Qwen2.5-Math-7B, and the commercial solver Mathematica. Reinforcement learning with Levenshtein Distance (LD) reward further improves accuracy. On the Hard set with beam size 70, FISolver achieves 63.7\%, nearly tripling Mathematica's 23.3\%.}
    \label{fig:main_results_bar}
\end{figure}

This paper focuses on two-dimensional first-order ODE systems as a controlled and practically meaningful testbed. This scope includes all scalar second-order ODEs after the standard transformation $\dot{x}=y$, covers important system families studied by specialized first-integral solvers, and allows rigorous symbolic verification at scale. We view higher-dimensional and parameterized systems as important extensions and discuss their computational challenges in Section~\ref{sec:discussion}.

Our main contributions are as follows:

\begin{itemize}
    \item \textbf{Data Generation:} We propose a novel ``Backward Generation'' algorithm (\cref{sec:backward_gen}) that systematically constructs large-scale datasets of (differential equation, first integral) pairs. This approach effectively overcomes the critical bottleneck of data scarcity, while symbolic verification provides a correctness guarantee for the retained pairs.
    \item \textbf{Foundation Model Training:} Leveraging the generated dataset, we perform LoRA-based~\cite{hu2022lora} supervised fine-tuning on a compact mathematical LLM (\cref{sec:sft}). This yields a foundational solver that substantially outperforms larger-scale general mathematical LLMs on standardized benchmarks while maintaining low computational overhead.
    \item \textbf{Domain Adaptation:} We introduce data synthesis (\cref{sec:synthesis}) and dataset blending (\cref{sec:blending}) techniques. These methods enable the model to adapt effectively to specific dynamical systems using only a sparse set of examples, demonstrating strong transferability to specialized domains.
    \item \textbf{Reinforcement Learning Alignment:} We design a mathematically grounded reward function incorporating Levenshtein Distance-based shaping and employ the GRPO algorithm~\cite{shao2024deepseekmath} for reinforcement learning (\cref{sec:rl}). This structured guidance significantly enhances the model's symbolic reasoning capabilities beyond what is achievable with supervised learning alone.
\end{itemize}

\section{Related Work}

\paragraph{First Integrals Discovery.}
First integrals, also known as conserved quantities or constants of motion in many physical systems, are fundamental to the study of dynamical systems. The study of first integrals is fundamentally established by Noether's theorem ~\cite{noether1971invariant} and advanced by analytical methods including Prelle-Singer ~\cite{prelle1981elementary}, Lie symmetry ~\cite{bluman2002symmetry}, and Liouvillian functions ~\cite{singer1992liouvillian}. While recent algorithmic tools like InSyDE ~\cite{braz2025new} improve solutions for specific systems, they remain limited by search complexity.
In machine learning, neural networks have been successfully applied to approximate conserved quantities numerically ~\cite{kasim2022constants,greydanus2019hamiltonian,chen2021neural,cranmer2020lagrangian}. However, these methods typically rely on numerical regression or brute-force search to recover symbolic expressions ~\cite{ha2021discovering,liu2022machine}. Unlike these unsupervised or numerical approaches, we propose a supervised LLM framework for direct symbolic derivation.

\paragraph{LLMs for Mathematical Reasoning.}
The Transformer architecture ~\cite{vaswani2017attention} has enabled powerful mathematical LLMs like DeepSeekMath ~\cite{shao2024deepseekmath} and Qwen2.5-Math ~\cite{yang2024qwen2}. While effective for general problems, they struggle with specialized symbolic tasks. Recent works have applied Transformers to linear algebra and dynamical systems ~\cite{charton2023learning,charton2021linear,cai2024transforming,alfarano2024global}. Distinct from traditional symbolic regression ~\cite{udrescu2020ai,biggio2021neural} or general equation discovery ~\cite{shojaee2024llm,kahlmeyer2025discovering}, our work focuses specifically on the intuitive discovery of first integrals via large-scale pre-training.

\paragraph{Data Generation and RL.}
Our backward generation strategy aligns with reverse-engineering paradigms used in other mathematical domains ~\cite{lu2024mathgenie,alfarano2024global}. To address data scarcity, we leverage synthetic bootstrapping ~\cite{yu2023metamath,luo2023wizardmath,gou2023tora} and high-quality data blending ~\cite{zhou2023lima,dong2024abilities,gunasekar2023textbooks}. Furthermore, we enhance reasoning via Reinforcement Learning with process supervision ~\cite{lightman2023let}. Specifically, we employ Levenshtein Distance ~\cite{levenshtein1965binary} as a shaped reward, adapting sequence editing principles ~\cite{ranzato2015sequence,petersen2019deep,paassen2018tree} to ensure symbolic precision in mathematical discovery.

\section{Preliminaries}

We consider a system of first-order $n$-dimensional differential equations
\begin{equation}
\frac{d\bm{x}}{dt}=\bm{f}(t,\bm{x})
\label{eqn1}
\end{equation}
where $\bm{f} \in C^{1}(D, \mathbb{R}^n), \;(t,\bm{x})\in D \subset \mathbb{R}^{1+n}$.

\begin{definition}[First Integral, \cite{arnold1992ordinary}]
A first integral of system (\ref{eqn1}) is a continuously differentiable function $V: D \to \mathbb{R}$ that is non-constant on any open subset of $D$, such that its total derivative vanishes along any solution trajectory $\bm{x}(t)$:
\begin{equation}
\frac{d}{dt}V(t,\bm{x}(t)) = \frac{\partial V}{\partial t} + \nabla_{\bm{x}}V \cdot \bm{f}(t, \bm{x}) = 0.
\label{eqn2}
\end{equation}
\end{definition}

Equivalently, for any particular solution $\bm{x}=\bm{x}(t)$ of the system (\ref{eqn1}), the function $V$ remains constant along that solution trajectory, i.e., $V(t,\bm{x}(t))\equiv C$ for some constant $C$. In two-dimensional cases, we often denote the state vector as $\bm x = (x,y)$ for clarity.

\begin{example}[Harmonic oscillator]
For $\dot{x}=y$ and $\dot{y}=-x$, the function $V=x^2+y^2$ is a first integral because
$\frac{dV}{dt}=2x\dot{x}+2y\dot{y}=2xy-2xy=0$.
It represents conservation of the oscillator's energy up to a constant factor.
\end{example}

\begin{theorem}[Local Existence of First Integrals, \cite{arnold1992ordinary}]
For any non-critical point $P_0 \in D$ of the system \eqref{eqn1}, there exists a neighborhood $U_0 \subset D$ in which the system possesses exactly $n$ functionally independent first integrals.
\end{theorem}

This local result allows time dependence and does not imply $n$ global, state-only conservation laws.

\begin{theorem}[General Solution via First Integrals, \cite{arnold1992ordinary}]
\label{theorem2}
Suppose $V_1(t,\bm{x}), \dots, V_n(t,\bm{x})$  are $n$ functionally independent first integrals of (\ref{eqn1}) in a neighborhood of a non-critical point. According to the Implicit Function Theorem, the solution $\bm{x} = \phi(t, \bm{C})$ obtained from the algebraic equations
$$
(V_1(t,\bm{x}), \dots, V_n(t,\bm{x})) = (C_1, \dots, C_n)
$$
constitutes the general solution of (\ref{eqn1}), where $\bm{C}$ is a vector of arbitrary constants.
\end{theorem}

This theorem establishes the fundamental connection between first integrals and solutions of a system: first integrals encode the implicit structure of solutions, and conversely, a complete set of first integrals fully determines the system's solution family.

\section{Method}

We propose FISolver, which learns to map a symbolic differential system directly to a first integral, while using symbolic verification to ensure correctness. To address the scarcity of verified training pairs, we combine solver-based forward generation with backward generation, which constructs differential systems from prescribed first integrals. Overall, FISolver comprises three main stages: (1) \textbf{data generation}, (2) \textbf{model training}, and (3) \textbf{inference and verification}, as illustrated in Figure~\ref{fig:first-integrals}.

\begin{figure*}
    \centering
    \includegraphics[width=1.0\textwidth]{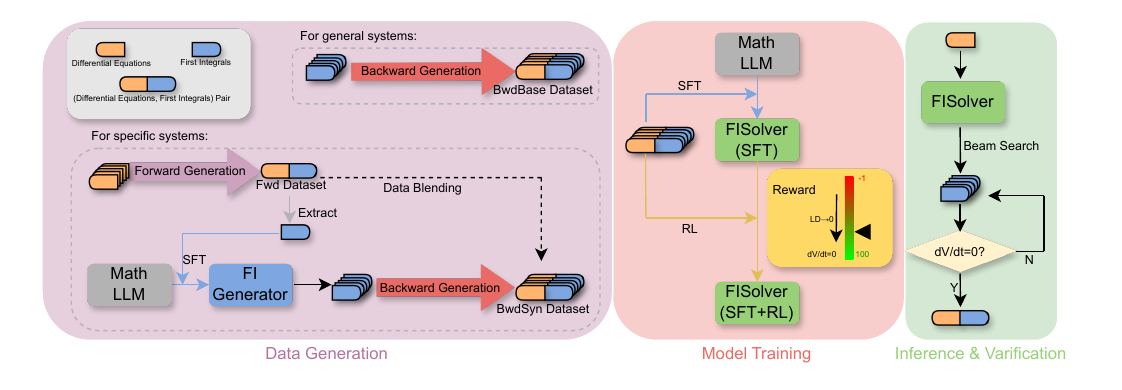}    \caption{FISolver comprises three main stages: (1) \textbf{data generation}, (2) \textbf{model training}, and (3) \textbf{inference and verification}.  \textbf{Data Generation:} We employ two complementary approaches. For general systems, we randomly generate first integrals and apply backward-generation to construct the BwdBase dataset, which trains models capable of solving diverse first integrals. For specialized systems, we leverage a small set of known first integrals (e.g., from forward-generation using specialized solving tools) to fine-tune a mathematical LLM, yielding a generative model FIGenerator that produces large quantities of synthetic first integrals. These are then inversely transformed via backward-generation to create the BwdSyn dataset, optionally augmented through data blending with forward samples. \textbf{Model Training:} We first perform supervised fine-tuning (SFT) on a pre-trained mathematical LLM, followed by reinforcement learning (RL) using a Levenshtein Distance (LD)-based reward function. \textbf{Inference and Verification:} Given a differential system, we use beam search to generate multiple candidate first integrals, which are then individually verified until a correct solution is identified.}
    \label{fig:first-integrals}
    \Description{Overview of the FISolver pipeline, including backward data generation, model training, and inference-time verification.}
\end{figure*}

\subsection{Data Generation}

\subsubsection{Forward Generation}

Some existing first integral solving tools can solve differential equations that meet specific criteria \cite{braz2025new}. We can leverage these tools to generate (differential equations, first integrals) pairs by first randomly generating systems of differential equations that satisfy the tool's solution conditions. If the tool successfully solves the system, we record the resulting pair as a data sample.

The data obtained using this method is often limited in quantity and belongs to a specific system. However, as long as these small amounts of data are fully utilized with some special techniques, they can greatly enhance the model's solving ability on specific systems.

\subsubsection{Backward Generation}
\label{sec:backward_gen}

The core idea of backward generation is to derive systems of differential equations from given first integrals. This section first introduces the algebraic principle of backward generation and then describes how the resulting pairs are filtered and used in later synthesis and blending stages.

Given $m$ functionally independent first integrals $(1\le m\le n)$ and $n-m$ non-constant differential equations, we can generate a pair consisting of a complete differential system and its $m$ prescribed first integrals as follows.

We define $\bm{V}=(V_1,V_2,...,V_m)^{\top}$ as the vector that includes all the given first integrals that satisfies Equation (\ref{eqn2}). $\bm{x}_a=(x_1,x_2,...,x_{n-m})^{\top}$ stands for variables whose differential equations are known and $\bm{x}_b=(x_{n-m+1},x_{n-m+2},...,x_n)^{\top}$ stands for the complementary set of variables. 
Let
$J_a=\frac{\partial \bm{V}}{\partial \bm{x}_a}\in\mathbb{R}^{m\times(n-m)}$
and
$J_b=\frac{\partial \bm{V}}{\partial \bm{x}_b}\in\mathbb{R}^{m\times m}$
denote the Jacobian matrices under the convention $(J_a)_{ij}=\partial V_i/\partial (x_a)_j$ and $(J_b)_{ij}=\partial V_i/\partial (x_b)_j$.
Then Equation~(\ref{eqn2}) can be written as

\begin{equation}
J_a\frac{d\bm{x}_a}{dt}
+J_b\frac{d\bm{x}_b}{dt}
+\frac{\partial\bm{V}}{\partial t}
=\bm{0}.
\label{eqn4}
\end{equation}

If matrix $J_b$ is invertible, the unknown differential equations can be uniquely determined by the following formula derived from Equation (\ref{eqn4}):

\begin{equation}
\frac{d\bm{x}_b}{dt}
=-J_b^{-1}
\left(
J_a\frac{d\bm{x}_a}{dt}
+\frac{\partial\bm{V}}{\partial t}
\right).
\label{eqn5}
\end{equation}

\begin{proposition}[Correctness of Backward Generation]
\label{prop:correctness}
Let $V_1, \dots, V_m$ be continuously differentiable functions, let $\frac{d\bm{x}_a}{dt}$ be given, and let $\frac{d\bm{x}_b}{dt}$ be any solution of the linear system~\eqref{eqn4}. Then $V_1, \dots, V_m$ are first integrals of the resulting complete system, i.e., $\frac{dV_i}{dt} \equiv 0$ for all $i = 1, \dots, m$.
\end{proposition}
\begin{proof}
Expanding $\frac{dV_i}{dt}$ via the chain rule (cf.\ Definition~1) and stacking all $m$ components yields
$\frac{d\bm{V}}{dt} = \frac{\partial \bm{V}}{\partial t} + J_a \frac{d\bm{x}_a}{dt} + J_b \frac{d\bm{x}_b}{dt}$,
which vanishes identically because $\frac{d\bm{x}_b}{dt}$ satisfies Equation~\eqref{eqn4}.
\end{proof}

\begin{remark}
Equation~\eqref{eqn4} determines $\frac{d\bm{x}_b}{dt}$ uniquely whenever $\det(J_b) \neq 0$, i.e., whenever the prescribed functions $V_1, \dots, V_m$ satisfy the local solvability condition of the Implicit Function Theorem with respect to the variables $\bm{x}_b$.
\end{remark}

Equation~(\ref{eqn5}) is used only to state the algebraic principle. In the actual construction process, we do not explicitly compute the symbolic matrix inverse. Instead, we treat Equation~(\ref{eqn4}) as a linear symbolic system in the unknown derivatives $\frac{d\bm{x}_b}{dt}$ and solve it with a symbolic solver (as detailed in Section~\ref{sec:datasets}). If the system has no symbolic solution, the sampled seed is discarded. If it is underdetermined, we keep a concrete solution only after instantiating the free components according to the sampling rules and verifying that all prescribed $V_i$ satisfy $\frac{dV_i}{dt}\equiv0$. This design improves robustness to singular or degenerate cases and is also computationally preferable, because explicit symbolic inversion tends to be substantially more expensive and can cause severe expression swell as the system dimension grows.

The proposed inversion framework introduces intrinsic flexibility into the data generation process. The seed components—comprising $m$ first integrals and $n-m$ explicit differential equations—can be sampled either stochastically from a general distribution to maximize diversity, or according to specific structural constraints to target specialized system families (as detailed in Section~\ref{sec:synthesis}).

Formally, the generation pipeline operates in three stages. First, we sample the initial set of first integrals and differential equations. Second, the full system is derived via the inversion logic defined in Equation (\ref{eqn5}). Finally, to ensure the generated samples possess sufficient dynamical complexity, we impose a non-triviality constraint: we reject any system containing decoupled variables, defined as equations where $\frac{dx_i}{dt} = f_i(x_i, t)$. This filtering step eliminates trivial univariate subsystems that provide little information about coupled dynamics.

\subsubsection{Data Synthesis}
\label{sec:synthesis}

While the random backward-generation procedure yields diverse data, it may lack coverage of specific structured families. To address this, we propose a data-synthesis strategy that leverages sparse examples to expand the training corpus, targeting two practical scenarios:

\begin{enumerate}
    \item \textbf{Low-Resource Adaptation.} For systems with only a few known first integrals, we use these sparse examples to fine-tune a first integral generator. This generator learns the underlying structural distribution to synthesize a large-scale corpus of similar expressions, transforming a few-shot problem into a data-rich supervised task.

    \item \textbf{Solver Enhancement.} When a specialized solver exists but has limited capability, we use it to generate a small set of ``seed'' solutions. By repeating the process in Scenario 1, we can train a LLM-based solver that generalizes beyond the original tool's limitations. This allows our model to discover first integrals for instances that the original solver could not handle.
\end{enumerate}

Both scenarios follow the same pipeline: learning from sparse data to synthesize massive pairs, thereby bridging the gap between existing knowledge and generalizable discovery.

The core idea is to train a generative mathematical model that produces a large number of first integrals resembling the available exemplars. Concretely, we fine-tune a mathematical LLM to act as a synthetic first-integral generator. The generator is trained using the fixed prompt \textit{``Generate a complex mathematical expression that satisfies a specific unknown rule''} as input, with the known first integrals serving as the ground-truth targets. To reduce computational cost and memory footprint we apply Low-Rank Adaptation (LoRA) and manage distributed mixed-precision training.

The complete synthesis pipeline is:
\begin{enumerate}
    \item Train a generative model that produces first integrals similar in form to the available exemplars.
    \item Sample $m$ synthetic first integrals from the generator for each target synthetic instance ($m$ depends on how many integrals are intended to be prescribed).
    \item Generate the remaining $n-m$ differential equations as required by the application specification.
    \item Use Equation (\ref{eqn4}) to solve for the remaining differential equations.
    \item Record the resulting (differential equations, first integrals) pairs and add them to the training corpus.
\end{enumerate}

Note that the synthetically generated differential systems may not satisfy every constraint of the original domain; nevertheless, in practice this synthesized data substantially improves the model's ability to solve the target family of systems.

Notably, this approach differs fundamentally from naive synthetic data generation where the model directly learns to imitate complete (differential equations, first integrals) pairs. Unlike such methods, which cannot guarantee correctness of generated samples, our tokenization combined with the backward-generation algorithm ensures the correctness of all synthesized data.

\subsubsection{Dataset Blending}
\label{sec:blending}

To further improve model performance in specific structured systems, we blend sparse forward-generated samples into the large-scale backward-generated dataset. This strategy is essential for distributional alignment: while backward generation provides broad symbolic diversity and verified supervision, its distribution is induced by the sampling process over first integrals and seed equations. Consequently, some structured systems may contain expression lengths, operator co-occurrences, or equation patterns that are underrepresented in the backward-generated corpus. Forward-generated samples act as distributional anchors for these target patterns. Since the model has already acquired fundamental symbolic reasoning from the synthetic corpus, exposure to even limited forward data allows it to adapt its internal representations toward the target distribution, effectively bridging the gap between large-scale synthetic training and structured-system discovery. In this sense, blending acts as a lightweight domain-adaptation mechanism rather than merely increasing the training-set size.

\subsection{Model Training}

\subsubsection{Supervised Fine-Tuning}
\label{sec:sft}

Instead of training a model from scratch, we employ LoRA-based supervised fine-tuning on a pre-trained, small-scale mathematical LLM. This approach is computationally efficient and allows the model to leverage its existing mathematical priors while adapting to the specific symbolic patterns of first integrals. During fine-tuning, the model is tasked with predicting ground-truth integrals from differential equations under a specific prompt. Depending on the composition of the training set, the resulting model can either be a general-purpose solver for diverse first integrals or a specialized solver optimized for particular classes of systems.

\subsubsection{Reinforcement Learning}
\label{sec:rl}

The RL training is initiated from the SFT model using the GRPO algorithm, with the same LoRA configuration for efficiency. Unlike SFT, which minimizes token-level cross-entropy against one reference expression, RL assigns sequence-level rewards after parsing and symbolic verification. This is important because many failed predictions are locally plausible and differ from a valid first integral by only a few tokens or operators, yet still violate $\frac{dV}{dt}\equiv0$. Therefore, instead of relying only on sparse binary rewards, we use Levenshtein-distance shaping to provide intermediate guidance toward verifiable first integrals.

The success of the RL process relies on a robust and mathematically grounded reward function, $R\left(\frac{d\bm{x}}{dt},{\bm{V}},\hat{V}\right) \in \mathbb{R}$, where $\hat{V}$ denotes a predicted candidate first integral. The reward function is mainly composed of two parts: Levenshtein distance reward $R_L({\bm{V}},\hat{V})$ and
structural and validity penalties $P(\hat{V})$ ($R_L, P \in [0,1]$). The total reward for a generated expression is a composite score, calculated as:

$$
    R\left(\frac{d\bm{x}}{dt},{\bm{V}},\hat{V}\right) = \begin{cases}
        R_{\text{max}} , \frac{d \hat{V} }{dt}\equiv0 \\
        \omega \cdot  R_L  , \   \frac{d \hat{V}}{dt} \not \equiv 0 \\
        -k \cdot  P, \ \hat{V} \text{ is   invalid}
    \end{cases}
$$

where $R_{\text{max}} ,\omega , k > 0$ and $\omega R_L  \in [0,R_{\text{max}}]$.

\paragraph{Levenshtein distance reward $R_L$.} The Levenshtein distance (minimum edit distance) $L$ is calculated between the generated expression and all reference expressions in Polish (prefix) notation. The minimum distance is used. Then the distance $L$ is mapped to the reward component $R_L=\exp(-k_L \cdot L) \ (k_L>0)$.

\paragraph{Structural and validity penalties $P$.} Candidate expressions are considered to be invalid if one of the following conditions is satisfied:
\begin{itemize}
    \item Syntactically invalid. The generated tokens do not form a valid Polish expression. For example, \texttt{+ x y z} has too many operands and cannot be parsed as a single binary expression. In this case, $P=1-\exp(-|E_k| \cdot k_S)$, where $E_k$ is the difference between the operand and the binary operator count minus $1$, since $E_k=0$ is a necessary condition for a valid Polish expression. 
    \item Detected with invalid symbols. For example, the generated sequence contains a variable \texttt{w}, which is outside the predefined vocabulary. In this case, $P=1$.
    \item Contains only one variable or does not contain any variables. In this case, $P=0$.
\end{itemize}

For each sample, the total reward is the maximum reward among all candidate expressions generated by the model.

Because prefix strings are not canonical, the Levenshtein reward may underestimate the quality of an algebraically equivalent expression that is far from the reference string. This limitation affects only the shaping signal. If the expression is a valid first integral, the symbolic verifier assigns $R_{\max}$ regardless of its string distance from the reference.

\subsection{Inference and Verification}

Verifying whether a predicted first integral is correct is straightforward. After simplifying $\frac{d\hat{V}}{dt}= \nabla _{\bm{x}} \hat{V} \cdot \frac{d\bm{x}}{dt}+\frac{\partial \hat{V}}{\partial t}$, the predicted first integral $\hat{V}$ is considered correct if and only if $\frac{d\hat{V}}{dt} \equiv 0$.

During the evaluation process, we employ \textbf{beam search} to generate multiple candidate first integrals. For a given system, we consider the model's prediction correct if at least one of the candidate first integrals generated via beam search is correct.

\section{Experimental Settings}

This section describes our experimental setup, beginning with the construction of three key datasets: \textbf{BwdBase} for general first integral discovery, \textbf{Fwd} for obtaining authentic target-domain examples from a specialized solver, and \textbf{BwdSyn} for scaling sparse examples from Fwd into a large training corpus via data synthesis. We detail the training of our Qwen2.5-Math-1.5B base model using LoRA-based supervised fine-tuning and GRPO-based reinforcement learning. Finally, we outline our evaluation framework, which utilizes beam search and symbolic verification of $\frac{dV}{dt} \equiv 0$ to ensure mathematical correctness.

\begin{table*}[t] 
    \centering
    \resizebox{\textwidth}{!}{
    \begin{tabular}{lrrrrrrrrrrrrr}
        \toprule
        \multirow{3}{*}{Test Sets} 
        & \multicolumn{7}{c}{bs=5} 
        & \multicolumn{5}{c}{bs=70} 
        & \multirow{3}{*}{Mathematica} \\
        \cmidrule(lr){2-8} \cmidrule(lr){9-13}
        
        & \multicolumn{5}{c}{FISolver-Base} 
        & \multicolumn{2}{c}{Other LLMs} 
        & \multicolumn{5}{c}{FISolver-Base} 
        & \\
        \cmidrule(lr){2-6} \cmidrule(lr){7-8} \cmidrule(lr){9-13}
        
        & SFT-only 
        & \shortstack{+ RL(NS)\\(LD)} 
        & \shortstack{+ RL(NS)\\(Bin)} 
        & \shortstack{+ RL(HS)\\(LD)} 
        & \shortstack{+ RL(HS)\\(Bin)} 
        & \shortstack{Qwen2.5-\\Math-1.5B} 
        & \shortstack{Qwen2.5-\\Math-7B} 
        & SFT-only 
        & \shortstack{+ RL(NS)\\(LD)} 
        & \shortstack{+ RL(NS)\\(Bin)} 
        & \shortstack{+ RL(HS)\\(LD)} 
        & \shortstack{+ RL(HS)\\(Bin)} 
        & \\
        \midrule
        
        Normal 
        & {43.2} 
        & \textbf{49.0} 
        & 45.1 
        & 41.8 
        & 37.0 
        & 12.9 
        & 36.6 
        & 72.7 
        & 74.2 
        & 49.5 
        & \textbf{74.7} 
        & 71.1 
        & 74.0 \\
        
        Hard   
        & {26.1} 
        & 31.0 
        & 29.3 
        & \textbf{33.5} 
        & 26.0 
        & 2.4  
        & 2.6 
        & 62.2 
        & 61.5 
        & 59.9 
        & \textbf{63.7} 
        & 63.3 
        & 23.3 \\
        
        \bottomrule
    \end{tabular}
    }
    \caption{Main results on first integral discovery. We compare FISolver-Base (in both SFT-only and RL-enhanced versions) against mathematical LLMs and the commercial solver Mathematica, and the table shows the accuracy (\%) of first integral prediction on two test sets: Normal and Hard.
    Our results demonstrate that: (1) SFT on backward-generated data significantly elevates base performance; (2) increasing beam size (bs) further unlocks the model's potential; (3) Guided RL with Levenshtein Distance (LD) reward provides a consistent boost over binary rewards (Bin); and (4) specialized RL training on different data distributions (NS: NormalSplit, HS: HardSplit) further optimizes performance for respective complexities. Notably, FISolver surpasses Mathematica on the Hard set by a wide margin and achieves research-level performance with a compact 1.5B parameter size.}
    \label{mainTable}
\end{table*}

\subsection{Datasets}
\label{sec:datasets}

We focus on first integrals of 2D, first-order, parameter-free ODE systems: $\frac{dx}{dt}=f(x,y,t),\frac{dy}{dt}=g(x,y,t)$. To ensure consistency and simplify structural parsing, expressions use Polish (prefix) notation (e.g., $-1/\sin(x + y - 2) \to$ \texttt{/ -1 sin + + - 0 2 x y}) processed via SymPy~\cite{meurer2017sympy}. This unambiguous representation eliminates the need for parentheses and aligns more effectively with the sequential generation nature of LLMs, reducing the model's cognitive burden in understanding complex operators.

Polish notation is syntactically unambiguous but not canonical: mathematically equivalent expressions may still have different prefix strings due to commutativity, associativity, or algebraic simplification. We therefore do not use string equality as the final correctness criterion. All expressions are converted through SymPy simplification before tokenization when possible, and final correctness is determined by symbolic verification of $\frac{d\hat{V}}{dt}\equiv0$.

In practice, rather than directly inverting the matrix as in Equation (\ref{eqn5}), we treat Equation (\ref{eqn4}) as a system of differential equations of $\frac{d\bm{x}_b}{dt}$ and solve it symbolically using SymPy. Although this approach may occasionally lead to degenerate systems, the resulting dataset is sufficiently rich to enable the model to develop strong capabilities for first integral discovery.

We construct three primary datasets:

The first dataset, \textbf{BwdBase}, serves as the foundation for general first integral discovery. It is generated via backward generation by randomly sampling first integrals to derive differential systems. Following~\cite{lample2019deep}, integrals contain 0--6 operators (arithmetic, exponentiation, elementary functions) with equal probability, using 70\% variables from $\{x, y, t\}$ and 30\% single-digit integers. This dataset comprises 410k samples (400k training, 10k validation). We further generate two test sets using identical generation procedures: \textbf{Normal} (same distribution) and \textbf{Hard} (challenging systems containing variable $t$ and nonlinear operators).

While BwdBase provides broad coverage, certain important system families produce first integrals of considerably greater complexity that fall outside this distribution. The following two datasets are constructed to enable adaptation to such specific complex systems.

The second dataset, \textbf{Fwd}, provides authentic target-domain examples for a specific family of rational second-order ODEs, serving as seed data for synthesis and blending. It is generated using InSyDE~\cite{braz2025new} to solve equations of the form $\frac{d^2x}{dt^2}=\frac{Ax+Bt+C}{Dx+Et+F}$, reformulated as 2D systems $\frac{dx}{dt}=y,\frac{dy}{dt}=\frac{Ax+Bt+C}{Dx+Et+F}$. We generate random systems ($A$--$F$ are single-digit integers) and apply InSyDE to them. InSyDE has three built-in first integral search strategies—Base, S2, and S3—where Base is computationally inexpensive while S2 and S3 employ much more sophisticated algorithmic techniques at higher computational cost. We utilize the S3 strategy to yield the forward dataset Fwd with 7k samples and a corresponding evaluation set \textbf{Fwd-Eval} with 1k samples.

The third dataset, \textbf{BwdSyn}, enables domain-specific adaptation by scaling sparse Fwd examples into a large training corpus. We train a generator on approximately 2k Fwd exemplars to produce 350k synthetic pairs (340k training, 10k validation) via backward generation. Here, we enforce $\frac{dx}{dt}=y$ to simplify inversion and ensure consistency with forward exemplars, while maintaining the correctness guarantees of backward generation.

\subsection{Models and Training}
Our implementation is built on Hugging Face~\cite{lhoest2021datasets} using transformers~\cite{wolf2020transformers} and trl~\cite{vonwerra2022trl}, employing Qwen2.5-Math-1.5B-Instruct~\cite{yang2024qwen2} as the base. During Supervised Fine-Tuning (SFT), we train three models: (1) \textbf{FIGenerator}, a generative model trained on a small subset of Fwd samples to produce synthetic first integrals; (2) \textbf{FISolver-Base}, trained on the BwdBase dataset; and (3) \textbf{FISolver-Syn}, trained on the BwdSyn dataset. We also experiment with blending forward samples into the training set, denoting such models as ``\textbf{FISolver + a\% Fwd},'' where $a$ indicates the proportion of forward data.

\textbf{SFT Training Configuration:} All models are trained on two RTX 2080 GPUs with a per-GPU batch size of 2 and gradient accumulation steps set to 16. We apply LoRA with rank $r=16$, $\alpha=32$, and dropout 0.05, targeting the query, key, value, and output projection layers in the multi-head attention mechanism. We apply a linear learning rate decay and set the weight decay to 0.01. For prediction models, we train for 1 epoch with an initial learning rate of $2 \times 10^{-4}$; for the generative model, we train for 20 epochs with an initial learning rate of $1 \times 10^{-4}$. Total training time is approximately 23 hours.

We subsequently apply RL fine-tuning to the SFT-trained FISolver-Base model using the GRPO algorithm. Two training scenarios are evaluated: (1) 5k randomly sampled examples from BwdBase (\textbf{NormalSplit}, \textbf{NS}), and (2) 1k examples selected to match the \textbf{Hard} distribution (\textbf{HardSplit}, \textbf{HS}). Models fine-tuned with RL are denoted with ``\textbf{+ RL}''.

We design two reward schemes to validate the utility of incorporating the Levenshtein Distance (LD) as a shaping reward:

\begin{itemize}
\item Binary Reward (Bin): A sparse reward where $\omega = 0$; the model receives a large positive reward only when it predicts a fully correct first integral. Otherwise, it receives no positive reward.
\item LD-based Reward (LD): A shaped reward where $\omega = 50$; the model receives a continuous reward proportional to the negative Levenshtein Distance between the prediction and the ground truth, encouraging incremental progress.
\end{itemize}

Both reward functions are set with $R_{\text{max}}=100$, $k_L=0.1$, and $k_S=0.01$. For models trained with NS and HS, the suffixes ``\textbf{(NS)}'' and ``\textbf{(HS)}'' are added, respectively. For models trained with the two different reward functions, the model names are suffixed with ``\textbf{(LD)}'' and ``\textbf{(Bin)}''.

\textbf{RL Fine-tuning Configuration:} Training occurs on a single RTX 2080 GPU. We maintain the same LoRA configuration as in SFT, with learning rate $3 \times 10^{-6}$, $\beta=0.01$, $\epsilon=0.02$, per-device batch size 1, and gradient accumulation steps 8, with 1 iteration per epoch. Training requires approximately 4 hours per 1k samples.

\subsection{Inference and Verification}
We evaluate using beam sizes of 5 or 70. For beam size 70, a 300-second timeout per expression prevents resource exhaustion; timeouts or memory overflows are recorded as failures.

\section{Results}

We evaluate FISolver against professional mathematical LLMs and commercial solvers like Mathematica. Our results show that FISolver significantly outperforms these baselines, particularly on complex nonlinear systems where commercial tools struggle. Furthermore, we demonstrate that Levenshtein-based reward shaping in RL and dataset blending techniques are critical for achieving research-level accuracy and effective domain adaptation.

\subsection{Baseline Comparison}

We evaluate the performance of FISolver through a comprehensive comparison with baseline models, commercial solvers, and large-scale state-of-the-art LLMs. The results, summarized in Table~\ref{mainTable}, Table~\ref{deepseek}, and Figure~\ref{fig:main_results_bar}, demonstrate the efficacy of our proposed pipeline.

\subsubsection{Fundamental Performance Improvement}

We first compare our model FISolver-Base against its pre-fine-tuned version (Qwen2.5-Math-1.5B-Instruct) and Qwen2.5-Math-7B-Instruct, a larger mathematical LLM of the same family. To ensure a fair comparison, we fix the beam size (bs) to 5 for all models. The results are presented in Table~\ref{mainTable}.

Under this setting, our fine-tuned model achieves prediction accuracies of 43.2\% and 26.1\% on the Normal and Hard test sets, respectively, significantly outperforming the original pre-fine-tuned model (12.9\% on Normal and 2.4\% on Hard). Moreover, our model surpasses the larger mathematics-specific LLM Qwen2.5-Math-7B-Instruct under the same beam size. These results indicate that fine-tuning a small number of parameters can substantially enhance the model's ability to solve first integrals—even exceeding the performance of larger mathematical models.

\subsubsection{Incremental Gain from Guided RL}

We use RL to optimize expression-level correctness rather than next-token likelihood. Table~\ref{mainTable} compares the SFT checkpoint with RL checkpoints trained under binary and LD-shaped rewards.

As shown in Table~\ref{mainTable}, using the LD-based reward leads to a notable improvement in accuracy: +5.8\% on the Normal set and +7.4\% on the Hard set when beam size is 5. In contrast, the binary reward brings only marginal gains. This suggests that the sparse binary signal provides insufficient learning guidance due to high variance, while the Levenshtein Distance offers a smoother, more informative learning signal that progressively guides the model toward the correct expression.

The disparity becomes more pronounced at beam size 70:  +RL \allowbreak  (NS)(Bin) achieves only 49.5\% on Normal, far below the SFT-only baseline of 72.7\%, whereas all other variants benefit substantially from the larger beam. We attribute this to output-distribution collapse induced by the sparse binary reward on the relatively homogeneous NormalSplit data---the model concentrates probability mass on a narrow set of expression patterns, causing the 70 beams to largely duplicate one another and negating the diversity advantage of beam search. The LD-shaped reward avoids this pathology by providing continuous gradient signal even for near-miss predictions, thereby preserving output diversity.

These results indicate that cross-entropy training alone does not fully exploit the model's potential. With a suitably designed reward function, RL fine-tuning can significantly deepen the model's understanding of first integrals and lead to measurable gains in solving accuracy. A concrete illustration is provided in Example~3 of \cref{sec:case_study}, where the SFT-only model identifies the key component $\log(x-y)$ but fails to correctly combine it with the time variable $t$; after RL with the LD-based reward, the model discovers the valid first integral $\log(x-y)/t$.

\subsubsection{Comparison with Commercial Solvers}

Given the absence of commercial software specialized in discovering general first integrals, we employ the \texttt{DSolve} function from Mathematica~\cite{Mathematica143} as our commercial baseline. While \texttt{DSolve} is primarily designed for solving systems of differential equations rather than directly finding first integrals, obtaining a general solution for a two-dimensional system is equivalent to identifying its first integrals (as established in Theorem \ref{theorem2}).  

On the Normal test set, Mathematica demonstrates robust performance with 74.0\% accuracy, reflecting its strength in handling standard symbolic problems. Remarkably, our FISolver+RL(HS)(LD) model (with beam size 70) achieves 74.7\% accuracy, slightly surpassing the commercial state-of-the-art even on this standard distribution. On this test set, FISolver and Mathematica achieve comparable performance. This is expected because Normal contains relatively simple systems for which deterministic symbolic routines are effective. However, Mathematica is designed primarily for solving differential systems rather than directly exposing first integrals; when it succeeds, its output is often a lengthy implicit or general solution from which the underlying conservation relation is not easy to read.

The contrast becomes starker on the Hard test set, which involves complex nonlinearities. Mathematica's performance drops significantly to 23.3\%, whereas our model maintains a high accuracy of 63.7\%. This demonstrates that commercial solvers rely on rigid algorithmic rules, whereas our data-driven approach generalizes better to challenging, non-standard systems.

As an analytical solver, Mathematica is faster when successful but can hang indefinitely on non-standard systems. Conversely, FISolver guarantees results within a predictable timeframe defined by its beam size. For large-scale automated discovery, this robustness and guaranteed completion outweigh the specialized speed of solvers with potentially unbounded execution times.

\subsubsection{Research-Level Evaluation}

 To assess whether our model reaches research-level performance, we compare our model FISolver-Base + RL(HS)(LD) with DeepSeek-V3.2-Exp~\cite{deepseekai2024deepseekv32} via its API—a model recognized for strong performance on research-grade mathematical problems (Table~\ref{deepseek}). To simulate real-world research conditions, we maximize our model's performance under reasonable time and computational budgets by using a beam size of 70 on the Hard dataset. Due to cost constraints, we compare our model with DeepSeek using a beam size of 1 for the latter, as it employs chain-of-thought reasoning leading to long responses. To make this limitation explicit, we report the average token cost per response and estimate the cost of 70 independent samples. Running 70 DeepSeek samples would require approximately $70\times2869.6$ output tokens per instance, which is outside our evaluation budget and would not match the local predict-and-verify setting. Since our model runs locally and DeepSeek is accessed via API, direct time comparison is infeasible; instead, we compare the average number of tokens consumed per sample as a proxy for efficiency.

 As shown in Table~\ref{deepseek}, despite DeepSeek's vastly larger parameter count (684.7B vs. 1.5B), its accuracy (42.8\%) is significantly lower than our model's 63.7\%. Additionally, the average token count per sample for DeepSeek is 2869.6, while our model generates only 23.0 tokens per candidate. Even with a beam size of 70, the total average tokens per sample (1611.8) is still lower than DeepSeek's single candidate. In many failed cases, DeepSeek either performs lengthy chain-of-thought reasoning without reaching a verifiable first integral and eventually exceeds the output-length limit, or terminates with an unsupported final expression that does not satisfy $\frac{dV}{dt}\equiv0$. This result confirms that our model achieves research-level performance in first integral exploration with a minimal parameter budget, and demonstrates that in problems where answer verification guarantees correctness, our Predict-and-Verify paradigm is more efficient than chain-of-thought reasoning.

\subsection{Adaptability to Specific Complex Systems}

\begin{table}[t] 
    \centering
    \begin{tabular}{lrr} 
        \toprule
        & Ours & DeepSeek-V3.2-Exp \\
        \midrule
        Accuracy (\%)          & \textbf{63.7} & 42.8 \\
        Parameter Scale       & {1.5B}          & 684.7B \\
        Beam Size             & 70            & 1    \\
        \addlinespace 
        Tokens Per Sample     & {1611.8}        & 2869.6 \\
                              & {\small (23.0 $\times$ 70)} & \\
        \bottomrule
    \end{tabular}
    \caption{Comparison between our FISolver-Base + RL(HS)(LD) model and DeepSeek-V3.2-Exp on the Hard test set. Despite significantly lower computational cost, our model achieves substantially higher accuracy, demonstrating that our approach reaches research-level mathematical reasoning capability.}
    \label{deepseek}
\end{table}

This subsection demonstrates that, through targeted training strategies, our model can develop strong proficiency in solving specific classes of complex structured systems. These results indicate that our model can be effectively deployed to tackle first integrals of specialized complex systems with significant research value.

The Fwd-Eval test set generated by InSyDE with strategy S3 contains first integrals of significantly greater complexity compared to those in the Normal and Hard sets. While BwdBase integrals contain at most 6 operators, the first integrals in Fwd-Eval often involve dozens of operators with deeply nested compositions of elementary functions, resulting in substantially longer expressions. Consequently, FISolver-Base—trained primarily on relatively concise first integrals—exhibits poor performance on this specialized domain. To bridge this gap, we conducted the following experiment using our data synthesis and blending approaches (\cref{sec:synthesis,sec:blending}).

We first train the prediction model solely on the BwdSyn dataset, then incrementally blend in Fwd samples and evaluate prediction accuracy. As shown in Table~\ref{InSyDE}, our model achieves 28.4\% accuracy using only synthetic data—surpassing both Base (9.8\%) and S2 (22.6\%) strategies of InSyDE. As more forward samples are added, accuracy further improves, reaching 35.4\% when 2\% forward samples (7,000 examples) are included. Since the Fwd-Eval test set was constructed using cases successfully solved by the S3 strategy, its 100\% accuracy serves as a validation of the data generation integrity and represents the performance ceiling for other methods on this specific family of systems.

These results support two key conclusions: (1) with only a small set of true (differential equations, first integrals) pairs, our approach can effectively overcome the OOD challenge and adapt to specialized classes of complex systems; (2) both proposed training strategies, synthetic data generation and forward sample blending, are empirically validated to be highly effective. A representative case is shown in Example~4 of \cref{sec:case_study}, where FISolver-Syn discovers a five-term first integral with a non-trivial fractional coefficient that InSyDE's Base and S2 strategies cannot find.

\begin{table}[t] 
    \centering
    \begin{tabular}{lr} 
        \toprule
        Methods & Accuracy (\%) \\
        \midrule
        FISolver-Syn            & {28.4} \\
        FISolver-Syn + 0.5\%Fwd & 30.5 \\
        FISolver-Syn + 1.0\%Fwd & 31.1 \\
        FISolver-Syn + 2.0\%Fwd & \textbf{35.4} \\
        \addlinespace[0.5em] 
        InSyDE-Base             & 9.8  \\
        InSyDE-S2               & 22.6 \\
        \bottomrule
    \end{tabular}
    \caption{Comparison with InSyDE on the Fwd-Eval test set. FISolver-Syn trained on BwdSyn substantially outperforms InSyDE's Base and S2 strategies. Moreover, as the proportion of forward samples (Fwd) in the training set increases, our model's accuracy continues to improve. Beam size is 70.}
    \label{InSyDE}
\end{table}

\subsection{Illustrative Examples}
\label{sec:case_study}

To further examine the behavior of FISolver beyond aggregate accuracy, we select several representative examples for detailed analysis. These cases illustrate four complementary aspects of the model: its ability to recover physically meaningful invariants, its capacity to discover conservation laws involving time-dependent nonlinear structures that frontier reasoning models fail to identify, the way LD-based reinforcement learning can correct typical failed predictions, and its effective domain adaptation to specialized system families where dedicated solvers struggle.

\noindent\textbf{Example 1: Stability analysis via energy conservation.}
In classical mechanics and control theory, determining whether a nonlinear system possesses a conserved quantity is a standard task for qualitative analysis: such an invariant constrains trajectories to its level sets, reducing the effective dimension of the dynamics without solving the ODE~\cite{arnold1992ordinary}. Consider the following time-dependent system:
\[
\begin{cases}
\dot{x}=\dfrac{xy+2y}{2t x^2+4tx+ty}\,, \\[6pt]
\dot{y}=\dfrac{-2x^2y-4xy}{2t x^2+4tx+ty}\,.
\end{cases}
\]
We feed its Polish-notation encoding to FISolver:

\begin{inputbox}
\texttt{\small [/ + * 2 y * x y + + * t y * * 2 t \textasciicircum{} x 2 * * 4 t x , / + * * - 0 4 x y * * - 0 2 y \textasciicircum{} x 2 + + * t y * * 2 t \textasciicircum{} x 2 * * 4 t x]
}
\end{inputbox}
\begin{outputbox}[FISolver-Base, bs=70]
\texttt{\small [+ y \textasciicircum{} x 2]}
\end{outputbox}

Converting the output back to standard notation yields $\hat V=x^2+y$, verified by $\frac{dV}{dt}=2x\dot{x}+\dot{y}=0$. Once identified, trajectories are constrained to the parabolic level sets $x^2+y=C$, reducing the two-dimensional dynamics to motion along one-dimensional curves and providing direct qualitative insight into the system's long-term behavior. This case shows that FISolver can output interpretable symbolic knowledge rather than merely optimizing a black-box metric.

\noindent\textbf{Example 2: Exponentially weighted conservation in population dynamics.}
In mathematical ecology, conserved quantities involving exponential weights arise in compartmental models where different populations experience distinct time-dependent rates~\cite{murray2002mathematical}. Such exponentially weighted sums recall the balance laws that govern redistribution across compartments in ecological and epidemiological settings. Consider:
\[
\begin{cases}
\dot{x}=\dfrac{ye^t(1-t)(x^2+1)}{t(x^2+1)+e^t}\,, \\[6pt]
\dot{y}=\dfrac{-y(x^2+1+e^t)}{t(x^2+1)+e^t}\,.
\end{cases}
\]

\begin{inputbox}
\texttt{\small [- - + / * y exp t + + t * t \textasciicircum{} x 2 exp t / * * y \textasciicircum{} x 2 exp t + + t * t \textasciicircum{} x 2 exp t / * * t y exp t + + t * t \textasciicircum{} x 2 exp t / * * * t y \textasciicircum{} x 2 exp t + + t * t \textasciicircum{} x 2 exp t , - - - 0 / y + + t * t \textasciicircum{} x 2 exp t / * y \textasciicircum{} x 2 + + t * t \textasciicircum{} x 2 exp t / * y exp t + + t * t \textasciicircum{} x 2 exp t]}
\end{inputbox}
\begin{outputbox}[FISolver-Base + RL(HS)(LD), bs=70]
\texttt{\small [+ x * y exp t]}
\end{outputbox}

Converting the output yields $\hat V=x+ye^t$, verified by $\frac{dV}{dt}=\dot{x}+\dot{y}e^t+ye^t\equiv 0$. The invariant states that the ``effective total''---where the second variable $y$ is weighted by the exponentially growing factor $e^t$---is conserved. This structure is analogous to compartmental balance laws in mathematical ecology~\cite{murray2002mathematical}: changes in $y$, amplified by the exponential weight $e^t$, are exactly compensated by changes in $x$, keeping the weighted total constant. Notably, when presented with this system, DeepSeek-V3.2-Exp failed to produce any candidate first integral despite extended chain-of-thought reasoning, whereas FISolver identified the compact invariant directly. This case demonstrates that FISolver can discover physically interpretable conservation laws involving time-dependent exponential structures that elude even frontier reasoning models.

\noindent\textbf{Example 3: RL-guided correction for non-autonomous systems.}
Logarithmic first integrals arise naturally in mathematical ecology and chemical kinetics. The celebrated Lotka--Volterra predator--prey model, for instance, possesses a conserved quantity of the form $\delta x - \gamma\log x + \beta y - \alpha\log y$~\cite{murray2002mathematical}. When such systems include time-varying coefficients, the resulting non-autonomous ODEs produce time-dependent invariants like $\log(\cdot)/t$ that are substantially harder to discover. Consider:
\[
\begin{cases}
\dot{x}=\dfrac{2(x-y)\log(x-y)}{2tx+3t}\,, \\[6pt]
\dot{y}=\dfrac{(-2x^2+2xy-x+y)\log(x-y)}{2tx+3t}\,.
\end{cases}
\]
The ground-truth first integrals are $V_1=x^2+x+2y$ and $V_2=\frac{\log(x-y)}{t}$.

\begin{inputbox}
\texttt{\small [+ / * * - 0 2 y log - x y + * 3 t * * 2 t x / * * 2 x log - x y + * 3 t * * 2 t x , - + + / * y log - x y + * 3 t * * 2 t x / * * - 0 2 \textasciicircum{} x 2 log - x y + * 3 t * * 2 t x / * * * 2 x y log - x y + * 3 t * * 2 t x / * x log - x y + * 3 t * * 2 t x]}
\end{inputbox}
\begin{errorbox}[FISolver-Base]
\texttt{\small [* t log - x y] ,\  [* \textasciicircum{} t 2 log - x y] ,\  [* x log - x y] ,\  $\ldots$}

\smallskip
{\scriptsize\color{gray} Converting these candidates: $t\log(x-y)$,\; $t^{2}\log(x-y)$,\; $x\log(x-y)$,\; $\ldots$ None satisfies $\frac{dV}{dt}\equiv0$.}
\end{errorbox}
\begin{outputbox}[FISolver-Base + RL(HS)(LD), bs=70]
\texttt{\small [/ log - x y t]}
\end{outputbox}

Converting the corrected output yields $\hat V=\frac{\log(x-y)}{t}$, verified by $\frac{d\hat V}{dt}\equiv0$. The SFT-only model had identified the key component $\log(x-y)$ but failed to correctly handle its relation with the time variable $t$. After RL with the LD-based reward (\cref{sec:rl}), the model successfully reorganizes these symbolic fragments into a globally valid expression. This case provides qualitative evidence that the shaped reward guides the model beyond local plausibility toward symbolic correctness, explaining the gains over binary reward in Table~\ref{mainTable}.

\noindent\textbf{Example 4: Generalized energy conservation in a traveling-wave system.}
When the forcing of a non-autonomous ODE depends on $t$ and $x$ only through a linear combination $\xi=t+x$, the system possesses a translational symmetry $(t,x)\to(t-a,x+a)$ that reduces the effective number of independent variables. Consider the rational second-order ODE $\ddot{x}=\frac{9t+9x+7}{2t+2x+5}$ reformulated as a first-order system:
\[
\begin{cases}
\dot{x}=y\,, \\[6pt]
\dot{y}=\dfrac{9t+9x+7}{2t+2x+5}\,.
\end{cases}
\]

\begin{inputbox}
\texttt{\small [y , / + + 7 * 9 t * 9 x + + 5 * 2 t * 2 x]}
\end{inputbox}
\begin{outputbox}[FISolver-Syn + 2\%Fwd, bs=70]
\texttt{\small [+ + + + y * 1/2 \textasciicircum{} y 2 * - 0 9/2 t * - 0 9/2 x * 31/4 log + + 5 * 2 t * 2 x]}
\end{outputbox}

Converting the output yields the five-term first integral
$\hat V = \tfrac{1}{2}y^2 + y - \tfrac{9}{2}(t+x) + \tfrac{31}{4}\log(2t+2x+5)$,
verified by $\frac{dV}{dt}\equiv 0$. Introducing the traveling coordinate $\xi=t+x$ with velocity $\dot{\xi}=1+y$, the first two terms can be rewritten as $\frac{1}{2}y^2+y=\frac{1}{2}(1+y)^2-\frac{1}{2}=\frac{1}{2}\dot{\xi}^2-\frac{1}{2}$, revealing that the invariant is an energy conservation law $\frac{1}{2}\dot{\xi}^2+U(\xi)=E$ in the $\xi$ coordinate, where the potential $U(\xi)=-\frac{9}{2}\xi+\frac{31}{4}\log(2\xi+5)$ has a singularity at $\xi=-\frac{5}{2}$ that acts as an impassable barrier dividing the phase space into disconnected invariant regions. This system belongs to the Fwd-Eval test set of rational second-order ODEs (\cref{sec:datasets}). InSyDE's~\cite{braz2025new} Base and S2 search strategies, both specifically designed for this class of equations, fail on this instance (Table~\ref{InSyDE}); only the computationally expensive S3 strategy succeeds. In contrast, FISolver-Syn, trained via our data synthesis and blending pipeline with only ${\sim}$2k seed examples (\cref{sec:synthesis,sec:blending}), identifies the correct expression including the non-trivial fractional coefficient $\frac{31}{4}$. This case demonstrates effective domain adaptation: a compact LLM-based solver can match the discovery capability of a specialized algorithmic tool while requiring substantially less computation.

\section{Discussion}
\label{sec:discussion}

We leverage LLMs to address the difficult and significant mathematical challenge of finding first integrals for differential equations. Our model, using the proposed data generation and fine-tuning methods, achieves research-level solving capability with minimal computational resources, successfully discovering first integrals for complex systems lacking general solutions.

The Backward Generation method is our central innovation, solving the data scarcity problem and enabling supervised learning for first integral discovery.

Our methodology is broadly applicable to first integral discovery and other mathematical problems. Key techniques include: (1) LoRA fine-tuning on existing mathematical LLMs for efficiency; (2) Data synthesis and mixing with forward-generated data to enrich the training set from a small number of samples; (3) Adopting a ``Predict-and-Verify'' solving paradigm; and (4) Employing Reinforcement Learning with Levenshtein Distance to enhance solving accuracy. 

A common challenge for ``Backward Generation'' is Out-of- \allowbreak Distribution (OOD) generalization. We overcome the base model's limitations for specific first integral types through data synthesis and dataset blending, allowing the model to adapt to new systems with minimal correct samples.


This work has several limitations. Our evaluation is limited to two-dimensional, first-order, parameter-free systems; higher-dimensional or parameterized systems may require larger models or stronger constrained decoding. Polish notation is syntactically convenient but non-canonical, so equivalent first integrals may receive dissimilar Levenshtein-based rewards. Finally, general-purpose LLM baselines are not tested with beam size 70 because decoding many long candidates is costly.

The implications of this research are twofold. Practically, it establishes the feasibility of combining symbolic computation with data-driven models to discover first integrals, impacting mathematics, physics, and engineering. In downstream applications, discovered first integrals directly enable stability certification via Lyapunov-like analysis, dimensionality reduction of complex dynamical systems, and the identification of hidden conservation laws that inform control design and physical model validation. Methodologically, our work highlights LLMs' potential to develop a ``super-intuition'' for bypassing tedious symbolic derivations, freeing researchers for higher-level conceptual work. We anticipate this research will bridge machine learning and symbolic mathematics, fostering interdisciplinary tools and advancing automated mathematical discovery toward greater reliability and applicability.

\begin{acks}
The corresponding author Shuai Li is supported by National Natural Science Foundation of China (92570111).
\end{acks}

\bibliographystyle{ACM-Reference-Format}

\balance

\section*{GenAI Usage Disclosure}
The authors used generative AI tools for language editing and drafting assistance. All technical claims, mathematical derivations, experimental results, data, code, and citations were verified by the authors.

\bibliography{references}

\end{document}